\documentclass[11pt]{article}
\PassOptionsToPackage{hidelinks}{hyperref}

\usepackage[final]{acl}

\usepackage{times}
\usepackage{latexsym}
\usepackage{amsmath}
\usepackage{amssymb}
\usepackage{amsthm}
\usepackage{booktabs}
\usepackage{graphicx}
\usepackage{subcaption}
\usepackage[hidelinks]{hyperref}
\hypersetup{
    pdftitle={HypEHR: Hyperbolic Modeling of Electronic Health Records for Efficient Question Answering},
    pdfauthor={Yuyu Liu, Sarang Rajendra Patil, Mengjia Xu, Tengfei Ma}
}
\usepackage{xurl}
\usepackage{tabularx}
\usepackage{xcolor}

\theoremstyle{plain}
\newtheorem{proposition}{Proposition}

\theoremstyle{remark}

\newcommand{\pms}[1]{{\scriptsize$\pm$\,#1}}

\usepackage[T1]{fontenc}
\usepackage[utf8]{inputenc}
\usepackage{needspace}
\usepackage{microtype}

\usepackage{inconsolata}

\usepackage{graphicx}
\graphicspath{{../}}

\usepackage{amsmath}

\title{HypEHR: Hyperbolic Modeling of Electronic Health Records for Efficient Question Answering\thanks{Accepted at Findings of ACL 2026}}

\author{
  Yuyu Liu\textsuperscript{1} \quad
  Sarang Rajendra Patil\textsuperscript{2} \quad
  Mengjia Xu\textsuperscript{2} \quad
  Tengfei Ma\textsuperscript{3} \\
  \textsuperscript{1}Department of Computer Science, Stony Brook University \\
  \textsuperscript{2}Department of Data Science, New Jersey Institute of Technology \\
  \textsuperscript{3}Department of Biomedical Informatics, Stony Brook University \\
  \texttt{\{yuyu.liu, tengfei.ma\}@stonybrook.edu} \\
  \texttt{\{sp3463, mx6\}@njit.edu}
}

\begin{document}
\maketitle
\begin{abstract}
Electronic health record (EHR) question answering is often handled by LLM-based pipelines that are costly to deploy and do not explicitly leverage the hierarchical structure of clinical data. Motivated by evidence that medical ontologies and patient trajectories exhibit hyperbolic geometry, we propose HypEHR, a compact Lorentzian model that embeds codes, visits, and questions in hyperbolic space and answers queries via geometry-consistent cross-attention with type-specific pointer heads. HypEHR is pretrained with next-visit diagnosis prediction and hierarchy-aware regularization to align representations with the ICD ontology. On two MIMIC-IV-based EHR-QA benchmarks, HypEHR approaches LLM-based methods while using far fewer parameters. Our code is publicly available at \url{https://github.com/yuyuliu11037/HypEHR}.
\end{abstract}
\section{Introduction}
%TODO: Add recent works related to EHR representation learning methods (before sit at two extreme)

Electronic health record (EHR) question answering (EHR-QA) aims to answer natural-language clinical questions over a patient’s longitudinal record~\cite{bardhan_question_2023}. For example, ``has patient been admitted to the emergency room on the first hospital visit'' or ``is there any microbiological test result on the current hospital visit for patient's blood culture?''~\cite{bae_ehrxqa_2023}. Recent datasets over MIMIC-III/IV have driven progress, but most methods sit at three extremes: (i) EHR representation learning methods, including sequential and graph-based models that encode temporal and heterogeneous clinical structures~\cite{miotto_deep_2016,li_behrt_2020,landi_deep_2020,chen_predictive_2024}, (ii) text-to-SQL or graph semantic parsers~\cite{wang_text--sql_2020,lee_ehrsql_2023,raghavan_emrkbqa_2021,bardhan_drugehrqa_2022}, and (iii) retrieval-augmented pipelines built on large language models (LLMs) such as GPT-3.5/4~\cite{kweon_ehrnoteqa_2024,elgedawy_dynamic_2024,wu_instruction_2024}. These approaches can be accurate, yet they are computationally heavy, hard to deploy under strict privacy constraints, and largely ignore the strong structural priors present in EHR data.

Prior research in EHR representation learning indicates that medical codes and longitudinal patient trajectories are intrinsically hierarchical, exhibiting properties that align closely with hyperbolic geometry~\cite{lu_learning_2019, beaulieu-jones_learning_2019}. While Euclidean embeddings distort tree-like structures, hyperbolic spaces can embed hierarchies with arbitrarily low distortion. Building on this insight, ~\citet{lu_self-supervised_2023} demonstrate that hyperbolic embeddings of the medical code hierarchy can improve temporal health event prediction, but their resulting patient representations are ultimately modeled in a Euclidean space. This raises a central research question: can a compact model, explicitly aligned with the intrinsic geometry of EHRs at the patient level, compete with billion-parameter LLMs in complex question answering?

\begin{figure}[t]
  \centering
  \includegraphics[width=\linewidth]{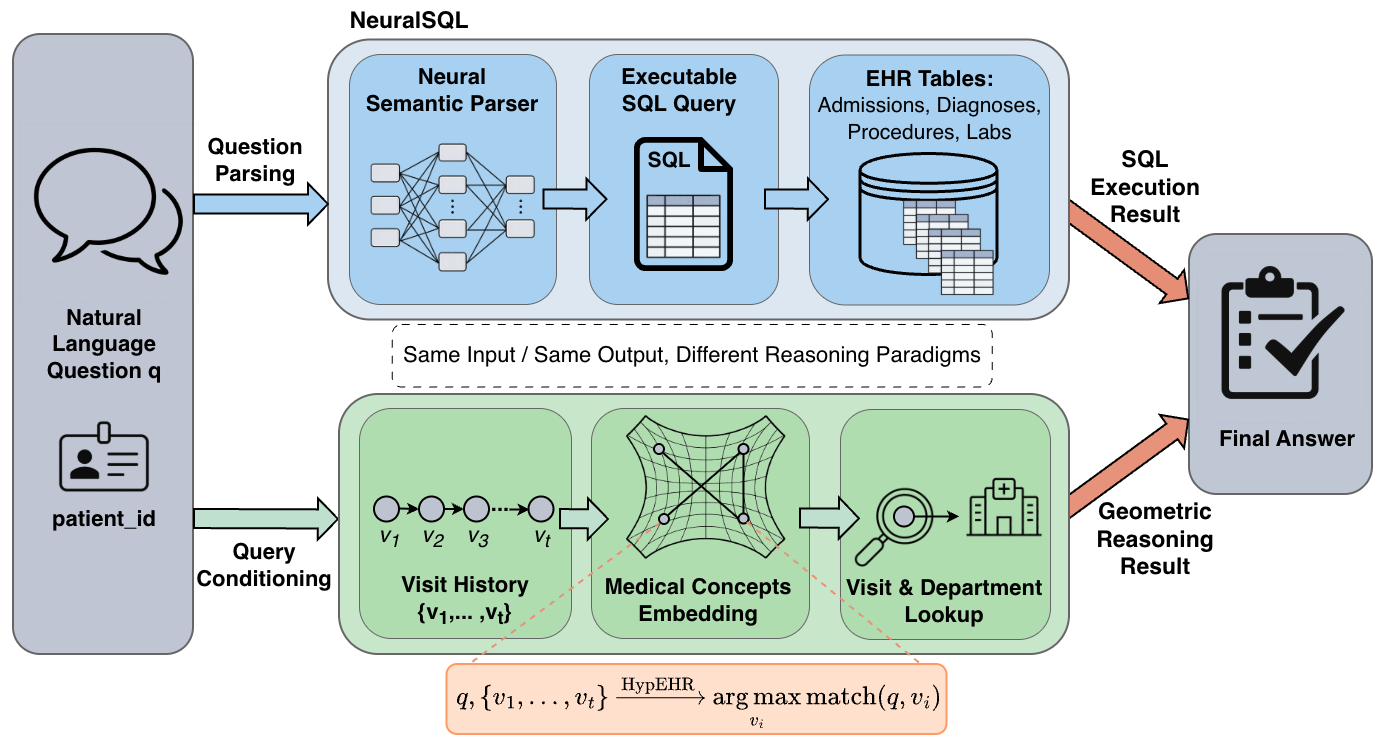}
  \caption{Comparison of workflows between text-to-SQL (\textit{top}) and our method (\textit{bottom}). The text-to-SQL-based methods typically rely on large-scale pretrained large language models to generate accurate SQL queries, whereas our method is specifically adapted to medical data, enabling comparable performance with a significantly smaller number of parameters.}
  \label{fig:compare}
  \vspace{-.2in}
\end{figure}

We address this question with \textbf{HypEHR} (\textbf{Hyp}erbolic modeling of \textbf{E}lectronic \textbf{H}ealth \textbf{R}ecords), a novel compact EHR-QA framework based on hyperbolic clinical sequence modeling. The comparison of pipelines between HypEHR
and text-to-SQL is shown in Figure~\ref{fig:compare}. The resulting model achieves performance comparable to large language models, while being orders of magnitude smaller (22M) than typical LLM-based pipelines (trillions of parameters) and thus more suitable for on-premise, privacy-conscious clinical settings.

\section{Methodology}

\begin{figure*}[t]
  \centering
  \includegraphics[width=1.0\textwidth]{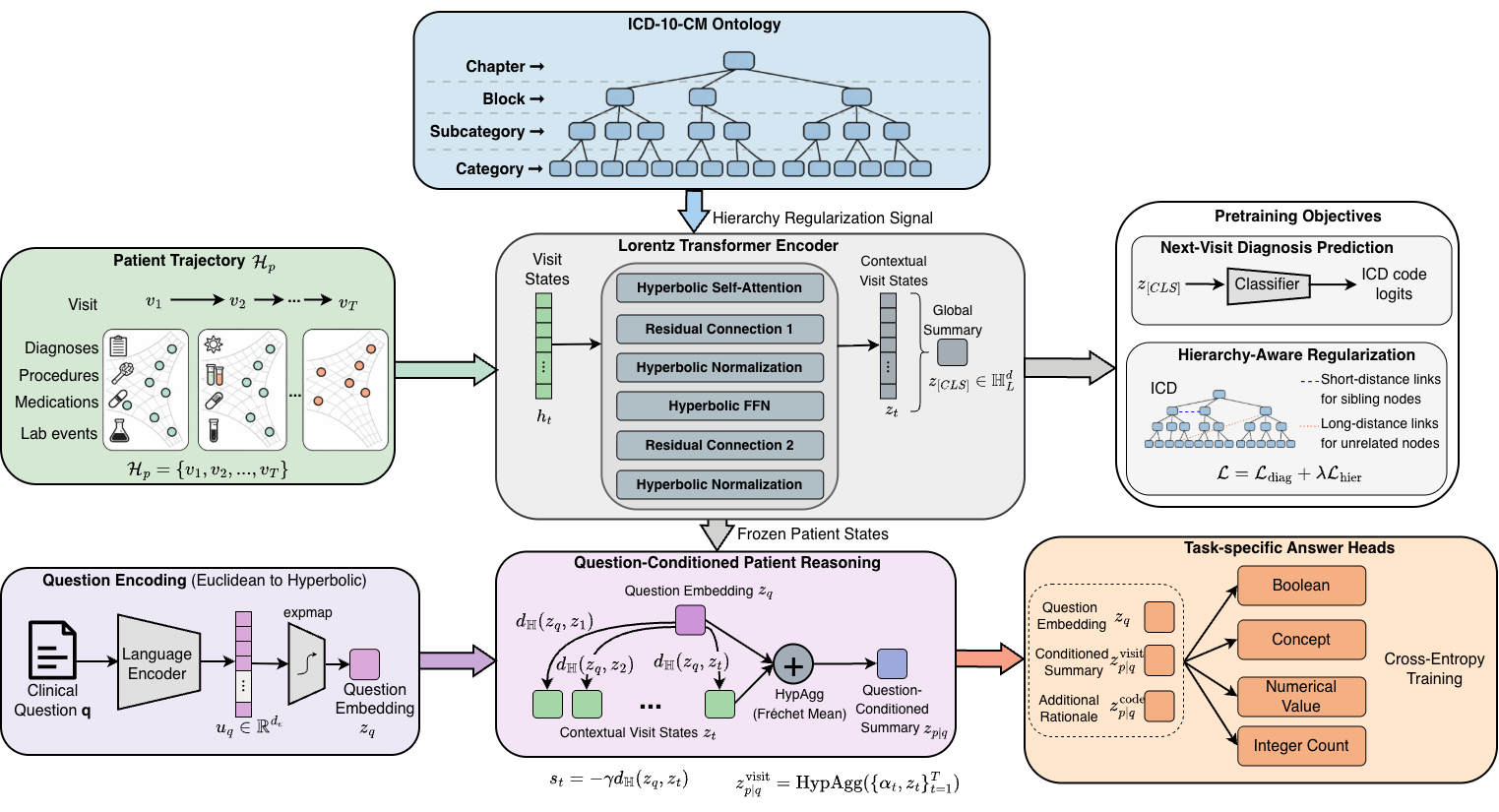}
  \caption{The overall framework of our proposed HypEHR.}
  \label{fig:model}
  \vspace{-.17in}
\end{figure*}

\subsection{Problem Definition and Model Overview}
Given a question $q$ and the visit history of patient $\mathcal{H}_p = \{v_1,\dots,v_T\}$, where each $v_t$ is associated with a set of medical concepts (including diagnosis codes, procedure codes, drug codes, admission time and other laboratory records), the model must look up the right visit, find the correct department, and return a clear answer. From the perspective of answer types, the common QA pairs can be categorized into four classes: boolean, concept, numerical, and integer count. 

We summarize the main modules of the HypEHR framework in Figure~\ref{fig:model} to provide an overview and the notations used throughout this paper in Table~\ref{tab:notation}. HypEHR consists of two stages. The first stage, \textit{patient encoder pretraining}, learns a hyperbolic patient encoder by joint targets of next-visit diagnosis prediction and hierarchy-aware regularization. The next stage, \textit{question-answer training}, trains answer type-specified heads using embeddings from the frozen patient encoder.

\subsection{Hyperbolic Clinical Sequence Encoder}
\label{sec:encoder}
We now present the first stage of our model --- \textit{patient encoder pretraining}. Each medical concept $c\in v_t$ is embedded into a Lorentzian hyperbolic manifold $\mathbb{H}^d_L$ via a learnable embedding function $e_c \in \mathbb{H}^d_L$. Within a visit, we aggregate the code embeddings using a hyperbolic attention mechanism to obtain a visit representation $h_t \in \mathbb{H}^d_L$. The sequence $\{h_t\}_{t=1}^T$ is then processed by a multi-layer Lorentz Transformer encoder, which adapts self-attention, residual connections, and normalization to the Lorentz manifold, yielding contextualized visit states $\{z_t\}_{t=1}^T$ and a global summary representation $z_{\text{[CLS]}} \in \mathbb{H}^d_L$. This global summary is then mapped to all diagnosis codes, producing next-visit diagnosis prediction loss $\mathcal{L}_{\text{diag}}$.

To encode hierarchical relationships in diagnosis code embeddings, we train patient encoder in a multi-task fashion with a hierarchy-aware regularizer $\mathcal{L}_{\text{hier}}$ using ICD code trie built by chapters $\to$ blocks $\to$ categories $\to$ subcategories, encouraging embeddings of codes that share ancestors in the ontology to be closer in hyperbolic distance than unrelated codes. We further decompose $\mathcal{L}_{\text{hier}}$ into a radial hierarchy term and a relative hierarchy term to jointly enforce depth ordering and local separation. Overall, the encoder parameters are optimized to minimize a joint objective 

\begin{equation}
\label{eq:total}
    \mathcal{L} = \mathcal{L}_{\text{diag}} + \lambda\mathcal{L}_{\text{hier}}
\end{equation}
where $\mathcal{L}_{\text{diag}}$ is the binary cross-entropy loss and 
\begin{align}
\mathcal{L}_{\text{hier}}
&= \mathcal{L}_{\text{rad}} + \mu\,\mathcal{L}_{\text{rel}}, \label{eq:hier} \\
\mathcal{L}_{\text{rad}}
&= \sum_{(p,c)\in\mathcal{P}}
\max\!\Big(0,\; \|e_p\|_{\mathbb{H}} - \|e_c\|_{\mathbb{H}} + \beta \Big), \label{eq:rad} \\
\mathcal{L}_{\text{rel}}
&= \sum_{(a,a^{+},a^{-})\in\mathcal{T}}
\max\!\Big(
0,\;
d_{\mathbb{H}}(e_a,e_{a^{+}}) \label{eq:rel} \notag \\
&\hspace{3.2em}
- d_{\mathbb{H}}(e_a,e_{a^{-}}) + \alpha
\Big)
\end{align}

with $\|e\|_{\mathbb{H}} := d_{\mathbb{H}}(e,\mathbf{o})$, and $\mathbf{o}$ denotes the origin of the Lorentz hyperboloid. Here $\mathcal{P}$ denotes parent--child pairs extracted from the ICD trie, and $\mathcal{T}$ denotes triplets where $a^{+}$ is an ancestor-related (e.g., parent/same-branch) code of $a$ and $a^{-}$ is a non-ancestor code; $\alpha,\beta>0$ are margin hyperparameters, and $\lambda,\mu>0$ 
balance the relative importance of each loss term. This yields a geometry-aware representation of patient trajectories that is later reused for downstream question answering.

\subsection{Hyperbolic EHR-QA Model}
Given a natural-language question $q$ about a patient $p$, our model combines a natural language encoder with the Lorentzian patient encoder described above. The question $q$ is first encoded by a biomedical pre-trained language encoder 
% (e.g., Bio\_ClinicalBERT\cite{alsentzer_publicly_2019}) 
into a Euclidean vector $u_q \in \mathbb{R}^{d_e}$, which is then projected into the hyperbolic manifold via an affine map followed by an exponential map at the origin, yielding a question representation $z_q \in \mathbb{H}^d_L$. We then perform hyperbolic cross-attention from $z_q$ over the sequence of visit states $\{z_t\}_{t=1}^T$: attention scores are defined as negative scaled hyperbolic distances $s_t = -\gamma d_{\mathbb{H}}(z_q, z_t)$ and normalized via softmax to obtain weights $\alpha_t$. A hyperbolic weighted Fréchet mean of visit states, $z^{\mathrm{visit}}_{p|q} = \mathrm{HypAgg}(\{\alpha_t, z_t\}_{t=1}^T)$, serves as a question-conditioned patient summary. We implement Hyperbolic Aggregation (HypAgg) as the exponential map at the origin of the weighted average of log-mapped points, which approximates the Riemannian barycenter on the Lorentz manifold. For fine-grained reasoning, we additionally apply a second-stage hyperbolic attention over code embeddings within the top-$k$ attended visits to construct a code-level rationale vector $z^{\mathrm{code}}_{p|q} \in \mathbb{H}^d_L$. Depending on the QA pair type, specialized answer heads consume $(z_q, z^{\mathrm{visit}}_{p|q}, z^{\mathrm{code}}_{p|q})$ to produce logits over the corresponding classes. The QA-specific components on top of the frozen language and patient encoders are trained using standard cross-entropy losses. Details of QA heads can be found in Appendix~\ref{sec:a_model}.

\vspace{-5pt}
\section{Experiments}
\vspace{-5pt}
\subsection{Experimental Setup}
\paragraph{Datasets and Tasks}
We adopt two representative EHR QA datasets for training and evaluation: \textbf{MIMIC-IV-Ext-Instr}~\cite{wu_instruction_2024} and the tabular subset of \textbf{EHRXQA}~\cite{bae_ehrxqa_2023}, and report the accuracy(\%) of the generated/retrieved answers. Besides, we also evaluate our model on four common clinical predictive tasks on \textbf{MIMIC-IV}~\cite{johnson_mimic-iv_2023}: (i) mortality prediction (MT), (ii) readmission prediction (RA), (iii) length-of-stay prediction (LOS), (iv) phenotype prediction (Pheno). AUPRC is adopted to evaluate the model’s performance on the above classification tasks.
\vspace{-.1in}

\paragraph{Baselines}
To comprehensively evaluate our proposed HypEHR, we adopt 6 representative methods as baselines for comparison from 3 main perspectives: (1) \textbf{text-to-SQL-based methods}: NeuralSQL~\cite{bae_ehrxqa_2023} with GPT-5.2~\cite{openai_introducing_2025} as the SQL parser, and a more light-weight version NeuralSQL-$l$ with code-smol2-text-to-sql~\cite{burtenshaw_burtenshawcode-smol2-text--sql_2024}  as the SQL parser, (2) \textbf{LLM-based methods}: Llemr~\cite{wu_instruction_2024}, EHRAgent~\cite{shi_ehragent_2024} and Llama-3-8B~\cite{aimeta_llama_2024}, where closed or open-source LLMs are used to process the question and structured patient history then generate answers, (3) \textbf{EHR representation learning-based methods}: a traditional patient sequence encoder RETAIN~\cite{choi_retain_2016} is used as the patient encoder in our workflow. Results are the mean and standard deviation of 5 runs over different random seeds.

More details about data processing (including task definitions), baseline implementations, and hyperparameter tuning could be found in Appendix~\ref{sec:a_dataset}, ~\ref{sec:a_baseline}, and ~\ref{sec:hparam}.

\vspace{-5pt}
\subsection{Experimental Results}
\paragraph{Question-answering Results}
Table~\ref{tab:model_performance} presents the accuracy for each baseline. Since EHRXQA formulates questions as executable SQL queries, methods that explicitly leverage LLMs for SQL generation, such as NeuralSQL and EHRAgent, naturally align with this paradigm and therefore achieve superior performance. Notably, our method achieves the best performance among all approaches that do not rely on large language model–based frameworks (e.g., GPT-5.2), highlighting the potential of hyperbolic embeddings.
%Compared to EHRXQA, MIMIC-IV-Ext-Instr questions rely more on clinical reasoning and free-form understanding, while EHRXQA is more schema-grounded and often directly translatable to executable SQL with deterministic answers. Consequently, EHRXQA has a more constrained solution space and is typically easier to evaluate objectively, which helps explain the generally higher model performance on this benchmark.

\begin{table}[h]
\centering
\begin{tabular}{lcccc}
\hline
\textbf{Model} & \textbf{EHRXQA} & \textbf{MIMIC-Instr}  \\
\hline
RETAIN         &     81.19 \pms{1.95} &  65.91 \pms{0.84}      \\
NeuralSQL      &     \underline{95.97 \pms{0.50}} &  75.17 \pms{0.73}   \\
NeuralSQL-$l$  &     86.72 \pms{0.97} &   67.85 \pms{0.85}    \\
Llama-3        &   82.88 \pms{1.38}  &     70.90 \pms{0.86}      \\
Llemr          &     87.25 \pms{0.77}    &   \underline{77.53 \pms{0.54}}    \\
EHRAgent       &     93.06 \pms{1.09}   &      74.16 \pms{0.56}       \\
\hline
HypEHR  &     89.53 \pms{0.60} &  76.02 \pms{0.41}   \\
\hline
\end{tabular}
\caption{Accuracy(\%) of models across two QA datasets. \textbf{MIMIC-Instr} denotes MIMIC-IV-Ext-Instr. \underline{NeuralSQL (EHRXQA)} and \underline{Llemr (MIMIC-Instr)} are official baselines based on large-parameter LLMs, and therefore serve as approximate upper bounds for current performance on these datasets.}
\label{tab:model_performance}
\vspace{-.16in}
\end{table}

\paragraph{Clinical Prediction Results}
To further assess whether our hyperbolic patient encoder learns generally useful representations beyond EHR-QA, we attach simple classification heads for standard clinical prediction tasks and compare its performance against baselines. The results can be found in Figure~\ref{fig:prediction}. HypEHR achieves the best performance on readmission prediction and phenotype prediction, and demonstrates performance comparable to the LLM–based baseline Llemr, on remaining tasks.

\begin{figure}[t]
  \centering
  \includegraphics[width=\columnwidth]{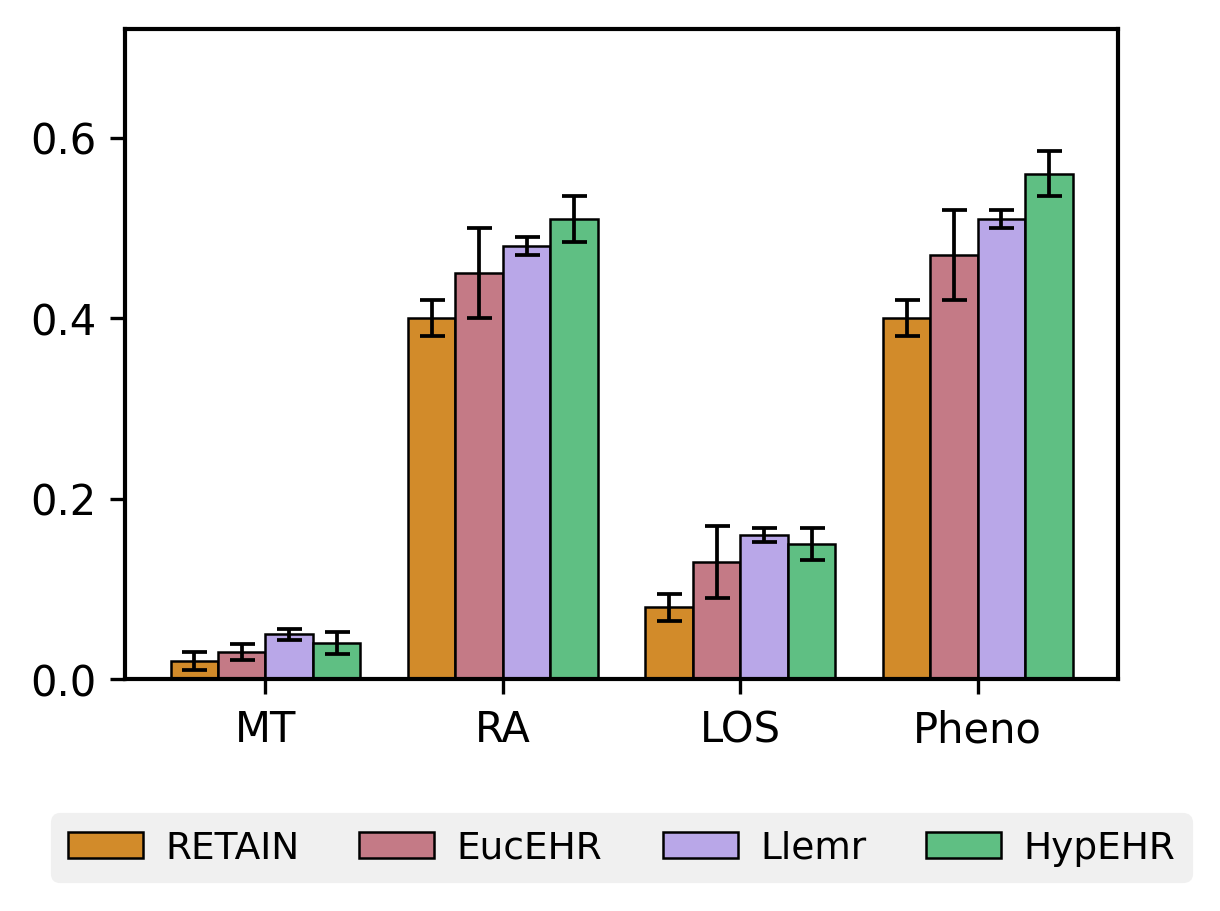}
  \caption{The AUPRC values of four models on the MIMIC-IV dataset.}
  \label{fig:prediction}
  \vspace{-.16in}
\end{figure}

% \begin{table}[h]
% \centering
% \begin{tabular}{lcccc}
% \hline
% Method & MT & RA & LOS & Pheno \\
% \hline
% RETAIN & 21.23 \pms{0.75} & 40.12 \pms{0.88} & 8.93 \pms{0.95} & 40.33 \pms{0.57} \\
% EucEHR &   &   &   &   \\
% Llemr  &   &   &   &   \\
% HypEHR &   &   &   &   \\
% \hline
% \end{tabular}
% \caption{AUPRC on MIMIC-IV across tasks.}
% \label{tab:mimic_iv_auprc}
% \end{table}
\vspace{-.1in}

\paragraph{Ablation Study}
To assess the contribution of each part attribute to model performance, we conduct an ablation study, evaluating HypEHR under different variants. Table~\ref{tab:ablation} reports this study. The results highlight the relative importance of different components within the model architecture. Pretraining of the patient encoder plays the most critical role, and even under the same pretraining setting, the Euclidean model performs substantially worse than the hyperbolic model on the test set. Although the hierarchy loss is not the primary contributor to overall model performance, it provides significant benefits in capturing and enforcing the hierarchical structure of codes.

\begin{table}[h]
\centering
\begin{tabular}{lccc}
\hline
\textbf{Model} & \textbf{EHRXQA} & \textbf{MIMIC-Instr} \\
\hline
w/o $L_{hier}$ &  82.72 \pms{3.41} & 70.38 \pms{0.54} \\ 
w/o pretraining  & 74.05 \pms{4.76} & 68.12 \pms{1.39} \\
EucEHR &  80.33 \pms{1.14} & 69.88 \pms{1.07} \\
\hline
HypEHR & 89.53 \pms{0.60} & 76.02 \pms{0.41}  \\
\hline
\end{tabular}
\caption{Ablation study on different variants of HypEHR. $\mathcal{L}_{\text{hier}}$ refers to Equation~\eqref{eq:hier}, pretraining denotes next-visit diagnosis prediction pretraining in Section~\ref{sec:encoder}, and EucEHR uses the same structure and pretraining task as HypEHR, but calculations are in Euclidean space.}
\label{tab:ablation}
\vspace{-.2in}
\end{table}

\paragraph{Geometry Analysis}
To test whether embeddings reflect the intrinsic ICD diagnosis hierarchy, we sample codes beginning with \texttt{I}, group them by tree depth, and compute for each code $c$ at depth $k$ an embedding radius---$\|e_c^{\text{Euc}}\|_2$ for the Euclidean baseline and $r_c^{\text{Lor}}=d_{\mathbb{H}}(o,e_c^{\text{Lor}})$ for the Lorentz model---then average within each level to obtain $\bar r_k^{\text{Euc}}$ and $\bar r_k^{\text{Lor}}$. Figure~\ref{fig:rvd} shows that Euclidean radii depend only weakly and noisily on depth (e.g., specific codes like \texttt{I21.9}/\texttt{I50.9} can have norms similar to \texttt{I10}), whereas the Lorentz model yields a clear monotonic increase, pushing deeper diagnoses farther from the origin; this radius--depth alignment suggests hyperbolic geometry better matches the tree-like expansion of the ICD hierarchy by allocating more capacity to fine-grained concepts near the boundary. 
%This validates our Proposition~\ref{prop:icd_lorentz} claiming that ICD codes admits low-distortion hyperbolic embeddings, and explains the accuracy gains of HypEHR over EucEHR (Table~\ref{tab:ablation}).

\begin{figure}[t]
  \centering
  \includegraphics[width=\columnwidth]{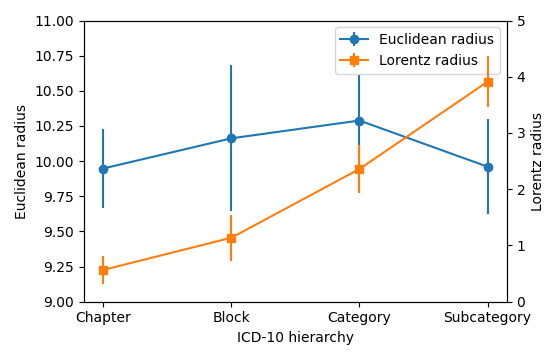}
  \caption{Comparison between Hyperbolic norms and Euclidean norms. 
  % Lorentz embeddings are measured by hyperbolic geodesic distance to the origin on the Lorentz (hyperboloid) manifold, while Euclidean embeddings are measured by the standard Euclidean \(L_2\) norm.
  }
  \label{fig:rvd}
\vspace{-.2in}
  
\end{figure}

\section{Conclusion}
In this work, we revisited EHR question answering from the perspective of data geometry. We introduced HypEHR, a lightweight Lorentz-based model that jointly encodes questions and clinical sequences. Experiments on MIMIC-IV-based QA benchmarks demonstrate comparative results with LLM baselines while using substantially fewer parameters.

\section*{Limitations}
Our approach also has several limitations. First, it relies on a preprocessing step that restructures each dataset so that answers fall into a small set of predefined categories (e.g., boolean, categorical concept, integer count, numeric values), which introduces additional engineering effort and computational overhead. Second, the current framework is restricted to discriminative answer types and does not naturally handle more open-ended, generative responses, limiting its extensibility to free-form clinical question answering. Third, hyperbolic neural networks are computationally more complex than their Euclidean counterparts, and the relative immaturity of public hyperbolic geometry libraries can lead to implementation challenges and potential instability in large-scale training.

\section*{Potential Risks and Ethical Considerations}
This work studies EHR question answering as a research problem and is not intended for direct clinical deployment. Potential risks include misinterpretation of model outputs if used for medical decision-making without proper clinical oversight, as well as biases inherited from retrospective EHR data. To mitigate these risks, our model is evaluated only on de-identified, publicly available datasets and does not provide diagnostic or treatment recommendations. We rely on the official de-identification procedures of these datasets and do not access or reconstruct any personally identifying information. We emphasize that such systems should be used as decision-support tools under human supervision rather than autonomous clinical agents. Future work should further investigate robustness, calibration, and fairness across patient subpopulations before real-world use.

\section*{Acknowledgments}
This work was supported in part by the DOE SEA-CROGS project (DE-SC0023191) and the AFOSR project (FA9550-24-1-0231). Research reported in this work was also partially funded through a Patient-Centered Outcomes Research Institute (PCORI) Award 102727. The views in this work are solely the responsibility of the authors and do not necessarily represent the views of the Patient-Centered Outcomes Research Institute (PCORI), its Board of Governors or Methodology Committee.

\bibliography{references,custom}

\appendix
\section{Notations}
\label{sec:a_notations}
This section summarizes the notation used throughout the paper for clarity and ease of reference.
Table~\ref{tab:notation} lists the definitions of all symbols appearing in the main text and appendices.

\begin{table*}
\centering
\small
\renewcommand{\arraystretch}{1.15}
\begin{tabularx}{\textwidth}{@{} l X @{}}
\toprule
\textbf{Symbol} & \textbf{Meaning} \\
\midrule
$p$ & Patient index. \\
$q$ & Natural-language question about patient $p$. \\
$\mathcal H_p=\{v_1,\dots,v_T\}$ & Visit history (trajectory) of patient $p$. \\
$T$ & Number of visits in the trajectory. \\
$v_t$ & The $t$-th visit. \\
$c \in v_t$ & A medical concept/code occurring in visit $v_t$. \\
$d$ & Hyperbolic embedding dimension (Lorentz model). \\
$d_e$ & Euclidean embedding dimension of the text encoder output. \\
$\mathbb H_L^d$ & $d$-dimensional hyperbolic space in the Lorentz (hyperboloid) model. \\
$\mathbb R^d$ & Euclidean Tangent Space $\mathbb H_L^d \cong \mathbb R^d$.  \\
$o$ & Origin point on the Lorentz hyperboloid (used for exp/log maps and radii). \\
$d_{\mathbb H}(\cdot,\cdot)$ & Hyperbolic geodesic distance on $\mathbb H_L^d$. \\
$e_c \in \mathbb H_L^d$ & Hyperbolic embedding of medical concept/code $c$. \\
h\textsubscript{t} $\in \mathbb H_L^d$ & Visit representation aggregated from code embeddings within $v_t$. \\
$\{z_t\}_{t=1}^T \subset \mathbb H_L^d$ & Contextualized visit states from the Lorentz Transformer encoder. \\
$z_{\mathrm{[CLS]}} \in \mathbb H_L^d$ & Global patient summary representation (CLS token/state). \\
$u_q \in \mathbb R^{d_e}$ & Euclidean question representation from the language encoder (pooled). \\
$u_i \in \mathbb R^{d_e}$ & Euclidean token embedding of the $i$-th question token. \\
$\tilde u_i = W u_i + b$ & Affine projection from Euclidean text space to tangent space at $o$. \\
$\exp_o(\cdot)$, $\log_o(\cdot)$ & Exponential and logarithmic maps between $\mathbb R^d$ and $\mathbb H_L^d$. \\
$z_i^q = \exp_o(\tilde u_i) \in \mathbb H_L^d$ & Hyperbolic embedding of question token $i$. \\
$z_q \in \mathbb H_L^d$ & Pooled hyperbolic question embedding (e.g., via Hyperbolic aggregation). \\
$\gamma$ & Cross-attention temperature/scale in scores $s_t=-\gamma d_{\mathbb H}(z_q,z_t)$. \\
$\alpha_t$ & Attention weight over visits (softmax-normalized). \\
$\mathrm{HypAgg}(\cdot)$ & Hyperbolic aggregation operator (approx.\ Fr\'echet mean on $\mathbb H_L^d$). \\
$z^{\mathrm{visit}}_{p|q} \in \mathbb H_L^d$ & Question-conditioned visit-level patient summary. \\
$z^{\mathrm{code}}_{p|q} \in \mathbb H_L^d$ & Question-conditioned code-level patient summary / rationale vector. \\
$L_{\mathrm{diag}}$ & Next-visit diagnosis prediction loss (binary cross-entropy). \\
$L_{\mathrm{hier}}$ & Hierarchy-aware regularization loss for ICD code embeddings. \\
$L_{\mathrm{rad}}$ & Radial hierarchy term encouraging depth ordering by hyperbolic radius. \\
$L_{\mathrm{rel}}$ & Relative hierarchy term enforcing ancestor vs non-ancestor separation. \\
$\lambda$ & Weight on hierarchy regularization in $L = L_{\mathrm{diag}} + \lambda L_{\mathrm{hier}}$. \\
$\mu$ & Weight on relative term in $L_{\mathrm{hier}} = L_{\mathrm{rad}} + \mu L_{\mathrm{rel}}$. \\
$\beta$ & Margin in radial hierarchy loss $L_{\mathrm{rad}}$. \\
$\alpha$ & Margin in relative hierarchy loss $L_{\mathrm{rel}}$. \\
$\mathcal P_{\mathrm{ICD}}$ & Set of (parent, child) pairs extracted from the ICD trie. \\
$\mathcal T_{\mathrm{ICD}}$ & Set of triplets $(a,a^+,a^-)$ for relative hierarchy training. \\
$\|e\|_{\mathbb H} := d_{\mathbb H}(e,o)$ & Hyperbolic radius (distance to origin). \\
$C_p=\{c_1,\dots,c_K\}$ & Per-patient candidate concept set for concept QA (plus optional null). \\
$c_{\mathrm{null}}$ & Learned ``no-answer'' pseudo-concept for concept QA. \\
$E=\{e_1,\dots,e_M\}$ & Candidate numeric events for a variable (e.g., lab test events). \\
$e_j$ & Numeric event with timestamp $t_j$, value $\nu_j$, and embedding $h_j^{\mathrm{val}}$. \\
$e_{\mathrm{null}}$ & Learned null event for no-answer numeric questions. \\
$e_c^{\text{Euc}}$ & Euclidean embedding of code $c$ (for the EucEHR baseline). \\
$e_c^{\text{Lor}}$ & Lorentz embedding of code $c$ (for HypEHR). \\
$r_c^{\text{Lor}}$ & Hyperbolic radius of code $c$, defined as $d_{\mathbb{H}}(o, e_c^{\text{Lor}})$. \\
$\bar{r}_k^{\text{Euc}}$ & Average Euclidean norm of codes at tree depth $k$. \\
$\bar{r}_k^{\text{Lor}}$ & Average hyperbolic radius of codes at tree depth $k$. \\
$\epsilon$ & Tolerance for matching a gold numeric value to candidate events. \\
$\mathcal I=\{j:\,|\nu_j-\nu|<\epsilon\}$ & Index set of events matching target value $\nu$. \\
$K_{\max}$ & Maximum discretized count for the count head. \\
\bottomrule
\end{tabularx}
\caption{Notation used throughout the paper. We use the Lorentz (hyperboloid) model $\mathbb H_L^d$ for hyperbolic embeddings and denote the hyperbolic distance by $d_{\mathbb H}$.}
\label{tab:notation}
\end{table*}

\section{Experiment Details}
\label{sec:a_experiment_details}

\subsection{Implementation Details}
\label{sec:implementation}

All experiments were conducted using 4 NVIDIA A100 GPUs with 80GB memory each. Models were trained with a batch size of 48 using the geoopt.optim.RiemannianAdam\footnote{\scriptsize \url{https://geoopt.readthedocs.io/en/latest/optimizers.html}} optimizer, and weight decay $1\times10^{-2}$. The learning rate was set to $3\times10^{-4}$ with linear warmup over the first 10\% of training steps, followed by cosine decay. Gradient norms were clipped to a maximum $\ell_2$ norm of 1.0. Dropout with rate 0.1 was applied to attention weights, feed-forward layers, and residual connections. Training used early stopping based on validation loss with a patience of 10 epochs for both patient encoder pretraining and question-answering head training. The maximum number of epochs was set to 250 for pretraining (approximately 2 hours) and 200 for question-answering heads (approximately 5 minutes), and the checkpoint with the best validation performance was selected for downstream evaluation. 

The Lorentz Transformer encoder consisted of 3 layers, each with 6 attention heads. The hyperbolic embedding dimension was set to 390, consistent with ~\citet{he_helm_2025}. The total number of trainable parameters in the full model was approximately 22 million, corresponding to a model size of 84 MB when stored in 32-bit floating point format. All models were implemented in PyTorch\footnote{\scriptsize \url{https://pytorch.org}} and leveraged Geoopt\footnote{\scriptsize \url{https://github.com/geoopt/geoopt}} for Riemannian optimization on the Lorentz manifold. Mixed-precision training (FP16) was enabled via NVIDIA Apex to reduce memory usage and improve throughput. For text encoder, we use Bio\_ClinicalBERT~\cite{alsentzer_publicly_2019}\footnote{\scriptsize \url{https://huggingface.co/emilyalsentzer/Bio_ClinicalBERT}.}.

\subsection{Data Preprocessing}
\label{sec:a_dataset}
%TODO: Add examples here

\paragraph{EHRXQA}
EHRXQA is a multi-modal EHR question answering dataset that links MIMIC-IV structured tables with aligned MIMIC-CXR chest X-ray images to generate Image-, Table-, and Image+Table QA pairs requiring both unimodal and cross-modal reasoning. In our experiments, we use the tabular subset of EHRXQA, and categorize these questions to Boolean Value, Count, Float Value, and Concept according to their answer type. The train/valid/test split is provided in json files in the dataset.

\paragraph{MIMIC-IV-Ext-Instr}
We use the Schema Alignment subset of MIMIC-IV-Ext-Instr and adopt the same preprocessing procedure as EHRXQA. In addition, our training, validation, and test splits follow those used in Llemr~\cite{wu_instruction_2024}.

\paragraph{MIMIC-IV}
We follow the data preprocessing process of ~\citet{chen_predictive_2024}, filtering out patients with less than two visits. ICD-9-CM codes are mapped to unique ICD-10-CM codes by General Equivalence Mappings (GEMs)\footnote{ \scriptsize The ICD-9-CM to ICD-10-CM General Equivalence Mappings (GEMs) are provided by the Centers for Medicare \& Medicaid Services (CMS) and made available via the National Bureau of Economic Research (NBER): \url{https://www.nber.org/research/data/icd-9-cm-and-icd-10-cm-and-icd-10-pcs-crosswalk-or-general-equivalence-mappings}.}. Statistics of processed data is shown in Table~\ref{tab:stats}.

\begin{table}[h]
\centering
\begin{tabular}{lcc}
\hline
\textbf{Dataset} & \textbf{MIMIC-IV}\\
\hline
\# of patients      &       14,155        \\
\# of visits          &       42,053        \\
Avg. \# of visits per patient      &        2.97           \\
Max \# of visits per patient  &         70     \\
\hline
\# of unique diagnoses   & 11,225  \\
\# of unique procedures   &  8,352 \\
\# of unique medicines   &  196 \\
\hline
\end{tabular}
\caption{Statistics of MIMIC-IV after pre-processing.}
\label{tab:stats}
\end{table}

\subsection{Baseline Implementations}
\label{sec:a_baseline}
\begin{itemize}
    \item RETAIN~\cite{choi_retain_2016} is a classical model for patient modeling. We use RETAIN to replace the patient encoder in our model, serving as a traditional baseline model.
    \item NeuralSQL is the standard baseline in EHRXQA~\cite{bae_ehrxqa_2023}, using gpt-3.5-turbo-0613~\cite{openai_introducing_2024} to generate SQL queries then retrieve answer from database. We replaced gpt-3.5-turbo-0613 by SOTA model GPT-5.2~\cite{openai_introducing_2025}. NeuralSQL-$l$ is a light-weighted version where a text-to-SQL specified small language model code-smol2-text-to-sql~\cite{burtenshaw_burtenshawcode-smol2-text--sql_2024} serves as the text-to-SQL parser.
    \item EHRAgent~\cite{shi_ehragent_2024} is an EHR question-answering intelligent agent equipped with a Python code interface and tool calling capabilities.
    \item Llama-3~\cite{aimeta_llama_2024} is a compact, open-weight large language model. We directly use the problem and related patient history codes as a prompt to generate the answer.
    \item Llemr~\cite{wu_instruction_2024}: Llemr is an instruction-tuned large language model framework that enables LLMs to process and interpret complex EHR data for diverse clinical question answering and predictive tasks. We use the pre-trained weights provided by the authors\footnote{\scriptsize \url{https://github.com/zzachw/llemr}.}.
\end{itemize}

\section{Additional Results}
\label{sec:a_results}

\paragraph{Accuracy of Each Type of Question}
For all question types (Boolean, single-concept, numerical, and count queries, including no-answer cases), we report exact match accuracy, i.e., the proportion of questions for which the predicted answer array exactly matches the gold answer array. Results are shown in Table~\ref{tab:accuracy_ehrxqa} and~\ref{tab:accuracy_instr}.

\begin{table*}[h]
\centering
\begin{tabular}{lccccc}
\hline
\textbf{Model} & \textbf{BL} & \textbf{CT} & \textbf{FL} & \textbf{CP} & \textbf{Overall} \\
\hline
RETAIN         &   84.36 \pms{3.46}    &  82.15 \pms{5.29}  &  78.22 \pms{4.23} & 80.74 \pms{1.20} & 81.19 \pms{1.95} \\
NeuralSQL      &    96.44 \pms{1.42}    &   95.71 \pms{0.96} & 94.97 \pms{0.58}  & 96.87 \pms{1.07} & 95.97 \pms{0.50} \\
NeuralSQL-$l$  &     88.40 \pms{2.37}     &  87.12 \pms{1.98}  & 84.53 \pms{2.25}  & 87.23 \pms{0.97} & 86.72 \pms{0.97}\\
Llama-3        &    86.40 \pms{2.25}     &  81.56 \pms{3.47}   & 81.09 \pms{2.88}   & 83.22 \pms{1.98} & 82.88 \pms{1.38} \\
Llemr        &     87.85 \pms{1.54}    &  88.96 \pms{2.01}  & 84.23 \pms{0.94}  &  88.21 \pms{1.46} & 87.25 \pms{0.77} \\
EHRAgent       &     92.46 \pms{1.14}    &  94.50 \pms{2.98}   & 93.26 \pms{1.84}  & 91.89 \pms{2.08} & 93.06 \pms{1.09} \\
\hline
HypEHR  &    91.54 \pms{1.74}   &  89.41 \pms{0.97}   &   90.48 \pms{0.78} & 87.05 \pms{1.33} & 89.53 \pms{0.60} \\
\hline
\end{tabular}
\caption{Accuracy(\%) for each question type on EHRXQA. Reported values are mean $\pm$ std. \textbf{BL}: Boolean Value, \textbf{CT}: Count, \textbf{FL}: Float Value, \textbf{CP}: Concept. Weights: (BL, CT, FL, CP) = (0.21, 0.26, 0.27, 0.26).}
\label{tab:accuracy_ehrxqa}
\end{table*}

\begin{table*}[h]
\centering
\begin{tabular}{lccccc}
\hline
\textbf{Model} & \textbf{BL} & \textbf{CT} & \textbf{FL} & \textbf{CP} & \textbf{Overall} \\
\hline
RETAIN         &   66.37 \pms{1.85}    &  64.29 \pms{2.00}  & 65.62 \pms{1.06}  & 67.33 \pms{1.94} & 65.91 \pms{0.84} \\
NeuralSQL      &  74.32 \pms{0.87}  &  76.03 \pms{1.14}  &  75.11 \pms{1.94} & 75.33 \pms{1.23} & 75.17 \pms{0.73} \\
NeuralSQL-$l$  &   69.64 \pms{1.56}    &  68.47 \pms{1.75}  & 67.86 \pms{1.64}  & 65.29 \pms{1.84} & 67.85 \pms{0.85} \\
Llama-3        &    70.35 \pms{1.74}     &  71.54 \pms{2.01}   &  69.98 \pms{1.47} & 72.09 \pms{1.75} & 70.90 \pms{0.86} \\
EHRAgent       &   74.35 \pms{1.25}     &   73.64 \pms{1.42}   & 72.54 \pms{0.89}  & 76.55 \pms{0.97} & 74.16 \pms{0.56} \\
Llemr        &    77.58 \pms{0.95}      &   76.21 \pms{1.04}     & 77.85 \pms{1.21}  & 78.34 \pms{0.96} & 77.53 \pms{0.54} \\
\hline
HypEHR  &    76.51 \pms{0.75}  &   77.32 \pms{0.98} & 74.46 \pms{0.90}  & 76.28 \pms{0.43} & 76.02 \pms{0.41} \\
\hline
\end{tabular}
\caption{Accuracy(\%) for each question type on MIMIC-IV-Ext-Instr. Reported values are mean $\pm$ std. \textbf{BL}: Boolean Value, \textbf{CT}: Count, \textbf{FL}: Float Value, \textbf{CP}: Concept. Weights: (BL, CT, FL, CP) = (0.25, 0.22, 0.30, 0.23).}
\label{tab:accuracy_instr}
\end{table*}

\section{Model Details}
\label{sec:a_model}
\subsection{Answer Types and Prediction Heads}
\label{sec:qa_heads_details}

In this work, we focus on \emph{per-patient} EHR-QA and restrict ourselves to the following answer categories:

\begin{itemize}
    \item \textbf{Boolean / existence questions:} answers of the form \([0]\) or \([1]\), denoting \textit{false} or \textit{true}. These questions ask whether a given condition, procedure, or event exists in the patient record.
    \item \textbf{Concept questions:} answers as a single-element string array, e.g., \([\text{``pneumonia''}]\) or \([\text{``low lung volumes''}]\). We align such strings to a discrete concept vocabulary (e.g., ICD, LOINC, or a curated set of findings) whenever possible.
    \item \textbf{Numeric value questions:} answers as a single floating-point value array, e.g., \([5.0]\), corresponding to a laboratory test result or scalar measurement.
    \item \textbf{Count questions:} answers as a single integer array \([k]\) with $k > 1$, e.g., the count of events or measurements matching a condition.
    \item \textbf{No-answer cases:} answers as an empty array \([]\), indicating that \emph{no} event in the patient record satisfies the query.
\end{itemize}
%\clearpage
We explicitly exclude questions whose answers are \emph{lists of patient identifiers}, such as \([10501557, 12215941, \dots]\), since these correspond to cohort-level retrieval rather than single-patient QA and are outside our scope.

\paragraph{Question-conditioned hyperbolic representations.}
All heads operate on a shared set of hyperbolic representations. Let $\{z_t\}_{t=1}^T \subset \mathbb{H}^d_L$ denote the Lorentzian visit-level embeddings of a patient's trajectory, and $z_{\text{[CLS]}} \in \mathbb{H}^d_L$ the global patient representation produced by our Lorentz Transformer encoder. A natural-language question $q$ is encoded by a text encoder into Euclidean token embeddings $\{u_i\}$, which are mapped into the Lorentz model via the exponential map at the origin:
\begin{equation}
    \tilde{u}_i = W u_i + b, 
    \quad
    z_i^q = \exp_o(\tilde{u}_i) \in \mathbb{H}^d_L.
\end{equation}
We obtain a pooled question representation $z_q \in \mathbb{H}^d_L$ via hyperbolic aggregation of $\{z_i^q\}$.

We then compute a question-conditioned patient summary at the visit level using hyperbolic attention:
\begin{equation}
    \alpha_t \propto \exp\big(-\gamma\, d_{\mathbb{H}}(z_q, z_t)\big)
\end{equation}
\begin{equation}
    z_{p|q}^{\text{visit}} = \mathrm{HypAgg}\big(\{\alpha_t, z_t\}_{t=1}^T\big)
\end{equation}
where $d_{\mathbb{H}}(\cdot,\cdot)$ is the Lorentzian distance and $\mathrm{HypAgg}$ is a Fr\'echet mean operator on $\mathbb{H}^d_L$. For concept-level reasoning, we further refine attention to the code-level within the top-$k$ attended visits, yielding a code-level summary $z_{p|q}^{\text{code}} \in \mathbb{H}^d_L$.

For all prediction heads, we map hyperbolic vectors back to the tangent space at the origin via the logarithmic map,
\begin{equation}
    \hat{z} = \log_o(z) \in \mathbb{R}^d,
\end{equation}
and feed $\hat{z}$ (optionally concatenated with other features) into Euclidean MLPs.

\subsubsection{Boolean / Existence Head}

This head handles questions whose answers are encoded as \([0]\) or \([1]\).

\paragraph{Input.}
We concatenate the question-conditioned patient representation and the question embedding in tangent space:
\begin{equation}
    h_{\text{bool}} = \hat{z}_{p|q}^{\text{visit}} \oplus \hat{z}_q \in \mathbb{R}^{2d}.
\end{equation}

\paragraph{Output.}
A small MLP produces logits $o \in \mathbb{R}^2$ for the labels ``no'' and ``yes'':
\begin{equation}
    o = \mathrm{MLP}_{\text{bool}}(h_{\text{bool}}),
    \quad
    p = \mathrm{softmax}(o),
\end{equation}
where $p_y$ denotes the predicted probability of label $y \in \{0,1\}$.

\paragraph{Loss.}
With ground-truth $y \in \{0,1\}$ derived from \([0]/[1]\), the loss is standard cross-entropy:
\begin{equation}
    \mathcal{L}_{\text{bool}} = - \log p_y.
\end{equation}
If a boolean-type question is annotated with an empty array \([]\), we normalize it to $y=0$ during preprocessing.

\subsubsection{Concept Head}

This head is used for questions whose answer is a single concept string, e.g., a diagnosis or finding.

\paragraph{Candidate set.}
For each patient, we construct a per-patient candidate set
\begin{equation}
    \mathcal{C}_p = \{c_1, \dots, c_K\},
\end{equation}
containing all concepts (codes or findings) appearing in that patient's EHR, plus an optional learned ``no-answer'' pseudo-concept $c_{\text{null}}$. Each candidate concept $c_j$ has a hyperbolic embedding $e_{c_j} \in \mathbb{H}^d_L$.

\paragraph{Input and scoring.}
We use the question-conditioned code-level representation $z_{p|q}^{\text{code}}$ and compute pairwise scores against each candidate:
\begin{equation*}
    \phi_j = \mathrm{MLP}_{\text{pair}}\big(\log_o(z_{p|q}^{\text{code}}) \oplus \log_o(e_{c_j})\big)
\end{equation*}
\begin{equation}
    s_j = w^\top \phi_j
\end{equation}

\paragraph{Output.}
We apply a softmax over all candidates:
\begin{equation}
    p_j = \frac{\exp(s_j)}{\sum_{k} \exp(s_k)}.
\end{equation}
At inference time, we select $\hat{c} = \arg\max_j p_j$ and output its associated string.

\paragraph{Loss.}
Let $c^{\*}$ be the target concept aligned from the answer string, and $j^{\*}$ its index in $\mathcal{C}_p$ (or the index of $c_{\text{null}}$ if the answer is empty):
\begin{equation}
    \mathcal{L}_{\text{concept-1}} = - \log p_{j^{\*}}.
\end{equation}

\subsubsection{Float Value Head}

This head handles questions whose answers are single numeric values, typically derived from laboratory tests or scalar measurements.

\paragraph{Event candidates.}
For a given variable (e.g., creatinine), we collect all matching events in the patient record:
\begin{equation}
    \mathcal{E} = \{e_1, \dots, e_M\},
\end{equation}
where each event $e_j$ has a timestamp $t_j$, a scalar value $\nu_j$, and a hyperbolic embedding $h_j^{\text{val}} \in \mathbb{H}^d_L$ (e.g., obtained from the corresponding visit state and variable identity). We additionally introduce a learned ``null event'' $e_{\text{null}}$ for no-answer cases.

\paragraph{Input and scoring.}
We use the question-conditioned visit-level representation $z_{p|q}^{\text{visit}}$ and compute pairwise scores:
\begin{equation*}
    \phi_j = \mathrm{MLP}_{\text{val}}\big(\log_o(z_{p|q}^{\text{visit}}) \oplus \log_o(h_j^{\text{val}})\big),
\end{equation*}
\begin{equation}
    s_j = w^\top \phi_j
\end{equation}

\paragraph{Output.}
We apply a softmax over all candidate events (including the null event):
\begin{equation}
    p_j = \frac{\exp(s_j)}{\sum_k \exp(s_k)}.
\end{equation}
The predicted answer is the value associated with the selected event, i.e., $\hat{v} = \nu_{\hat{j}}$ where $\hat{j} = \arg\max_j p_j$.

\paragraph{Loss.}
We treat numeric value prediction as a pointer-selection problem. Given a ground-truth value $\nu^{\*}$, we align it to one or more events in $\mathcal{E}$:
\begin{equation}
    \mathcal{I} = \{ j \mid |\nu_j - \nu^{{\*}}| < \epsilon \},
\end{equation}
for a small tolerance $\epsilon$. If at least one matching event exists, the loss is a multi-positive log-loss:
\begin{equation}
    \mathcal{L}_{\text{value}} = - \log \sum_{j \in \mathcal{I}} p_j.
\end{equation}
If no event in $\mathcal{E}$ matches the target value (or the gold answer is an empty array), we set $\mathcal{I} = \{j_{\text{null}}\}$ to select the null event.

\subsubsection{Count Head}

The count head is responsible for questions that ask for the number of events, visits, or occurrences satisfying a condition.

\paragraph{Input.}
We again use the question-conditioned visit-level representation:
\begin{equation}
    h_{\text{count}} = \hat{z}_{p|q}^{\text{visit}} \oplus \hat{z}_q.
\end{equation}

\paragraph{Output.}
We discretize counts into $\{0, 1, \dots, K_{\max}\}$, where $K_{\max}$ is chosen based on the empirical distribution (e.g., a high percentile). The head outputs logits $o \in \mathbb{R}^{K_{\max}+1}$:
\begin{equation}
    o = \mathrm{MLP}_{\text{count}}(h_{\text{count}}),
    \quad
    p = \mathrm{softmax}(o),
\end{equation}
where $p_k$ denotes the predicted probability of count $k$.

\paragraph{Loss.}
For a ground-truth count $k^{\*}$ (clipped to $K_{\max}$ if necessary), we use cross-entropy:
\begin{equation}
    \mathcal{L}_{\text{count}} = - \log p_{k^{\*}}.
\end{equation}
When the original answer is an empty array for a count-type question, we normalize it to $k^{\*} = 0$ during preprocessing.

\subsubsection{Overall Objective}

For each question, exactly one head is activated based on the parsed answer type. The total QA loss aggregates the head-specific losses together with auxiliary pretraining and geometry-aware regularization terms, where only the relevant terms are present for each sample. During QA training, only head-specific losses are optimized; pretraining losses are inactive due to encoder freezing.

\section{Hyperparameter Tuning}
\label{sec:hparam}

We tune hyperparameters on the validation split of each dataset.
Our training has two stages.
In Stage~1 we train the hyperbolic patient encoder with the objective
$\mathcal{L}=\mathcal{L}_{\mathrm{diag}}+\lambda \mathcal{L}_{\mathrm{hier}}$,
where $\mathcal{L}_{\mathrm{hier}}=\mathcal{L}_{\mathrm{rad}}+\mu \mathcal{L}_{\mathrm{rel}}$.
In Stage~2 we freeze the patient encoder and train the QA heads for each answer type. Best-found hyperparameter values are \underline{underlined} in each setting.

\paragraph{Stage~1: Patient Encoder Pretraining}
We select hyperparameters by minimizing the validation value of $\mathcal{L}$.
We tune:
(i) the number of Lorentz Transformer encoder layers $L$;
(ii) the hierarchy loss weight $\lambda$;
(iii) the relative-term weight $\mu$ in $\mathcal{L}_{\mathrm{hier}}$;
(iv) the margin parameters $\alpha$ and $\beta$ used in $\mathcal{L}_{\mathrm{rel}}$ and $\mathcal{L}_{\mathrm{rad}}$; (v) the learning rate for pretraining, and (vi) the hidden dimension $d$.
We run a grid search over the following sets:
$L\in\{{\underline{3}},5,8\},\lambda\in\{0,0.1,\underline{0.5},1.0,2.0\}$,
$\mu\in\{0.25,\underline{0.5},1.0\}$,
$\alpha\in\{0.1,\underline{0.2},0.5,1.0\}$,
$\beta\in\{0.1,\underline{0.2},0.5,1.0\},\eta\in\{10^{-5},3\cdot 10^{-5},\underline{10^{-4}},3\cdot 10^{-4}\},d\in\{130,260,\underline{390},520\}$.
After selection, we train the patient encoder once on the union of the training and validation sets for the same number of epochs, and we keep the final checkpoint for QA training.

\paragraph{Stage~2: Question Answering Heads}
Stage~2 uses the frozen patient encoder and a frozen text encoder, and trains only the QA-specific modules.
We tune hyperparameters by maximizing validation exact-match accuracy on each QA dataset.
We tune:
(i) the cross-attention distance scale $\gamma$ in
$\alpha_t \propto \exp(-\gamma\, d_H(z_q,z_t))$;
(ii) the number of top-$k$ attended visits used for the optional code-level attention;
and (iii) the learning rate for QA-head training.
We search $\gamma\in\{0.5,\underline{1},2,5\}$, $k\in\{1,2,\underline{4},8\}$, and $\eta_{\mathrm{QA}}\in\{10^{-5},\underline{3\cdot 10^{-5}},10^{-4}\}$.
For all other settings (batch size, optimizer type, and training epochs) we keep the same values across runs to isolate the effect of the tuned parameters.

\paragraph{Random seeds}
For the final reported numbers, we rerun training with five random seeds $(42,24,33,55,67)$ and report the mean and standard deviation.

\section{Related Work}
\label{sec:a_related_work}
\subsection{EHR Question Answering}
Prior work on EHR question answering (EHR-QA) spans unstructured clinical notes, structured patient records, and multimodal combinations. For note-centric QA, emrQA~\cite{pampari_emrqa_2018} constructs question-answer pairs from clinical annotations, while EHRNoteQA~\cite{kweon_ehrnoteqa_2024} targets patient-specific, multi-note reasoning grounded in real clinical queries. Structured QA is studied in emrKBQA~\cite{raghavan_emrkbqa_2021}, which maps questions to executable logical forms over EHR knowledge bases, as well as text-to-SQL benchmarks such as EHRSQL~\cite{lee_ehrsql_2023}. Multimodal datasets integrate notes and tables, including DrugEHRQA~\cite{bardhan_drugehrqa_2022} and EHRXQA~\cite{bae_ehrxqa_2023}. In parallel, retrieval-focused benchmarks and models—such as CliniQ~\cite{zhao_cliniq_2025}, DR.EHR~\cite{zhao_drehr_2025}, and retrieval-augmented methods like RGAR~\cite{liang_rgar_2025}—demonstrate that effective EHR-QA critically relies on accurate patient-specific evidence retrieval.

\subsection{Hyperbolic Neural Networks}
Transformer-style models in hyperbolic geometry largely build on the idea that hierarchical or power-law structure can be represented more naturally in negatively curved spaces~\cite{ravasz_hierarchical_2002, nickel_poincare_2017, yang_hyperbolic_2024}. Early work, such as Hyperbolic Attention Networks~\cite{gulcehre_hyperbolic_2018}, introduced hyperbolic variants of the attention mechanism and demonstrated how Transformer-like attention can be reformulated beyond Euclidean dot products. Subsequent efforts involved moving from hyperbolic attention inside an otherwise Euclidean Transformer toward more complete architectures. Hypformer~\cite{yang_hypformer_2025} proposes a Transformer defined end-to-end in the Lorentz model and further develops linear-time hyperbolic self-attention for scalability, enabling billion-scale graph processing.
More recent advances include HyLiFormer~\cite{li_hyliformer_2025}, which introduces hyperbolic linear attention for efficient hierarchical sequence modeling, and HELM~\cite{he_helm_2025}, which trains fully hyperbolic LLMs using a mixture of curvature experts and hyperbolic multi-head latent attention to align geometric representations with semantic hierarchies. Parameter-efficient adaptation methods 
such as HypLoRA~\cite{yang_hyperbolic_2024} enable hyperbolic fine-tuning of pre-trained models with up to 13\% improvement on mathematical reasoning tasks. Along this line, HyperGuide~\cite{liu2026hyperguide} exploits hyperbolic guidance to make multi-step reasoning in large language models more efficient. Beyond Transformer architectures, hyperbolic geometry has been integrated with state-space models for efficient sequence modeling: Hierarchical Mamba (HiM)~\cite{patil_hierarchical_2025} combines Mamba's linear-time complexity with learnable hyperbolic curvature for hierarchical reasoning, while HMamba~\cite{zhang_hmamba_2025} applies hyperbolic geometry to sequential recommendation with curvature-normalized discretization. Recent surveys have begun to categorize emerging hyperbolic deep learning architectures~\cite{peng_hyperbolic_2022, he_hyperbolic_2025, patil_hyperbolic_2025} across domains, tasks, and hyperbolic implementation approaches. However, despite these advances, the application of hyperbolic geometry to clinical question answering over structured EHR data remains largely unexplored.

\section{Geometry Hypothesis Validation}
\label{sec:a_math}
We justify our hypothesis about the EHR sequence using a Lorentzian hyperbolic space as the representation space for ICD-10-CM codes in our patient model by proposing the following theorem:

\begin{proposition}[Hyperbolic suitability of the ICD-10-CM hierarchy]
\label{prop:icd_lorentz}
Let $\mathcal{C}$ be the set of ICD-10-CM diagnosis codes used in our
study, and let $d_T$ be the tree metric induced by the official
parent--child hierarchy (each non-root code has a unique parent by
truncating its code string to a more general prefix). Consider the
$d$-dimensional Lorentzian hyperbolic space
\[
\mathbb{H}^d_L
\;=\;
\bigl\{ x \in \mathbb{R}^{d+1} : \langle x,x\rangle_L = -1,\; x_0 > 0 \bigr\},
\]
where $\langle \cdot,\cdot\rangle_L$ is the Minkowski bilinear form,
equipped with the induced hyperbolic distance $d_{\mathbb{H}}$.
Then:
\begin{enumerate}
  \item The metric space $(\mathcal{C},d_T)$ is $0$-hyperbolic in the
        sense of Gromov (i.e., a metric tree).
  \item For any $\varepsilon > 0$ there exists a dimension $d \ge 2$
        and an embedding $\varphi : \mathcal{C} \to \mathbb{H}^d_L$
        such that for all $u,v \in \mathcal{C}$,
\begin{equation}
\label{eq:quasi-isometry}
\begin{aligned}
(1-\varepsilon)\, d_T(u,v)
&\le d_{\mathbb{H}}\bigl(\varphi(u),\varphi(v)\bigr) \\
&\le (1+\varepsilon)\, d_T(u,v)
\end{aligned}
\end{equation}

\end{enumerate}
In particular, the ICD-10-CM hierarchy admits a low-distortion
embedding into the Lorentz model of hyperbolic space, so hyperbolic
distances between code embeddings can faithfully reflect their
hierarchical separation.
\end{proposition}

The geometric justification for this theorem lies in the exponential growth of volume in hyperbolic space, which naturally accommodates the exponential expansion of nodes in a hierarchy (chapters $\to$ blocks $\to$ categories).
 
\begin{proof}[Proof]
(1) The ICD-10-CM tabular list organizes diagnosis codes into a rooted
hierarchy (chapters $\to$ blocks $\to$ categories $\to$ subcategories),
with each non-root code having a unique parent obtained by truncating its prefix. Taking $\mathcal{C}$ as vertices and connecting each code
to its unique parent yields a connected, acyclic, rooted graph, hence
a simplicial tree $T$. Endowing $T$ with the path metric $d_T$ makes
$(\mathcal{C},d_T)$ a geodesic metric tree. By the standard
characterization of geodesic metric trees as precisely the
$0$-hyperbolic geodesic spaces
~\cite{gromov_hyperbolic_1987, bridson_basic_1999},
$(\mathcal{C},d_T)$ is $0$-hyperbolic.

(2) Results on embeddings of tree metrics into hyperbolic space show
that any finite tree $(\mathcal{C},d_T)$ admits, for every
$\varepsilon > 0$, a $(1+\varepsilon)$-bilipschitz embedding into the
hyperbolic plane $\mathbb{H}^2$; see, for example, Sarkar's
construction of low-distortion Delaunay embeddings of trees in the
hyperbolic plane~\cite{van_kreveld_low_2012}. Concretely, there exists
$\psi : \mathcal{C} \to \mathbb{H}^2$ such that for all $u,v \in
\mathcal{C}$, quasi-isometry \eqref{eq:quasi-isometry} holds.

The Lorentz hyperboloid model $\mathbb{H}^d_L$ is isometric to other
standard models of hyperbolic space (such as the Poincar\'e ball and
half-space models) via smooth bijections that preserve geodesic
distance \cite{bridson_basic_1999, nickel_poincare_2017, ganea_hyperbolic_2018}.
Extending $\psi$ to dimension $d \ge 2$ and composing with such an
isometry yields an embedding $\varphi : \mathcal{C} \to \mathbb{H}^d_L$
satisfying the same bilipschitz bounds. This establishes item (2) and
completes the proof.
\end{proof}

\end{document}